\documentclass[11pt,a4paper]{article}

\usepackage[comma]{natbib}
\usepackage{color}
\usepackage{graphicx}
\usepackage{pdfsync}

\usepackage{amsmath,amssymb,amsthm,amsopn}
\usepackage[ruled]{algorithm2e}

\usepackage[mathscr]{euscript}

\usepackage{pgfplots}

\pgfplotsset{compat=1.8}

\newtheorem{theorem}{Theorem}
\theoremstyle{remark}

\renewcommand{\today}{\begingroup
\number \day\space  \ifcase \month \or January\or February\or
March\or April\or May\or June\or July\or August\or September\or
October\or November\or December\fi \space  \number \year \endgroup}

\theoremstyle{plain}

\newtheorem{teor*}{Teorema}

\theoremstyle{definition}

\title{Explicit agreement extremes for a $2\times2$ table\\ with given marginals}

\author{Jos\'e E. Chac\'on\footnote{Departamento de
Matem\'aticas, Universidad de Extremadura, E-06006 Badajoz, Spain. E-mail:
{\tt jechacon@unex.es}}}

\begin{document}

\maketitle

\begin{abstract}
\noindent The problem of maximizing (or minimizing) the agreement between clusterings, subject to given marginals, can be formally posed under a common framework for several agreement measures. Until now, it was possible to find its solution only through numerical algorithms. Here, an explicit solution is shown for the case where the two clusterings have two clusters each.
\end{abstract}

%\medskip
%\noindent {\em Keywords:} misclassification error distance, adjusted Rand index, external clustering evaluation, confusion matrix,

\newpage

\section{Introduction}

Given two different clusterings of a data set, many measures have been proposed to quantify their degree of concordance. A recent review of a representative number of them can be found in \cite{M16}. These measures are usually categorized into three classes: those based on inspecting the assignments of data pairs in both clusterings, those involving some cluster matching between the two clusterings, and those relying on information theoretic criteria. This paper concerns the first one of these classes. In fact, some of the most popular and widely used similarity measures, such as the Rand index, the Jaccard index, or the Fowlkes-Mallows index, belong to this class of pair-based similarities, but it should be noted that there is a plethora of them, as explored in \cite{ANM06}, \cite{W08} or \cite{WV19}.

Precisely, when studying the Rand index \cite{MA84} noted that this statistic does not take into account the possibility of agreement by chance. \cite{HA85} suggested a general formulation to correct any of these indices for chance, which consists in substracting from the index its expected value when the clustering labels are assigned at random (with the constraint that the number of clusters and their sizes are fixed), followed by a normalization that ensures that the resulting corrected index still attains a value of 1 for identical clusterings. Namely, the adjusted version of any index is given by
$$\frac{\text{index}-E[\text{index}]}{\text{maximum index}-E[\text{index}]},$$
where ``maximum index'' is usually taken to be equal to 1, since that is the most general upper bound for these indices.

However, \cite{HA85} also noted that another (perhaps more adequate) bound that could be used is the maximum of the index given the fixed marginals (i.e., the cluster sizes of each of the two clusterings). Unfortunately, they refer to the problem of finding the maximum index value subject to the marginal constraint as ``a very difficult problem of combinatorial optimization.'' Nevertheless, some progress has been made since then. \cite{M92} pointed out that for some of the pair-based similarities, the problem is equivalent to finding the maximum of the sum of the squares of the contingency table. And \cite{BS08}, and later \cite{SHB15}, proposed a binary integer program and a heuristic algorithm, respectively, to obtain an exact solution numerically.

Here, on the contrary, the focus is on finding explicit expressions for the confusion matrix configurations that maximize the agreement between two clusterings, given the marginals. Due to the aforementioned difficulty of the problem, this study is restricted to the simplest case of $2\times2$ contingency tables. Even if the problem is not solved in its greatest generality, it should be hoped that the explicit form of the solution shown here for this case may inspire further research on the topic and serve as a first step towards a possible solution for clusterings of arbitrary size in the future. Moreover, any of the above similarity measures can be readily transformed into a semimetric by substracting it from 1 (see \citealp{Ch19}, for a detailed study of the semimetric thus obtained from the adjusted Rand index), and then it is also of interest to find its maximum value or, equivalently, the minimum possible agreement according to the index. That is, to discover the contingency table configuration, with given marginals, corresponding to the two most disparate clusterings.

Section \ref{sec:2} introduces the necessary notation to state the problem, and also includes the main result that shows its explicit solution. The extra advantage of having an explicit form for the solution over the numerical one is that it allows gaining intuition to understand the problem more deeply, as shown in the examples in Section \ref{sec:3}. Finally, some hints for the case of arbitrary clusterings are given in Section \ref{sec:4}.

\section{Problem statement and solution}\label{sec:2}

Given two clusterings $\mathscr C=\{C_1,\dots,C_r\}$ and $\mathscr D=\{D_1,\dots,D_s\}$ of a data set $\mathcal X=\{x_1,\dots,x_n\}$, let us denote by $n_{ij}=|C_i\cap D_j|$ the number of observations that are assigned to cluster $C_i$ in $\mathscr C$ and to cluster $D_j$ in $\mathscr D$. All the information about the concordance between $\mathscr C$ and $\mathscr D$ is collected in the $r\times s$ confusion matrix $\mathbf N=(n_{ij})$, also known as contingency table. The row and column marginals of $\mathbf N$ are determined, respectively, by the vectors $(n_{1+},\dots,n_{r+})$ and $(n_{+1},\dots,n_{+s})$, where $n_{i+}=\sum_{j=1}^sn_{ij}=|C_i|$ and $n_{+j}=\sum_{i=1}^rn_{ij}=|D_j|$.

For any pair of observations $x_k$, $x_\ell$ with $k\neq\ell$, there are four possibilities regarding their group assignments according to $\mathscr C$ and $\mathscr D$: a) they are in the same cluster both in $\mathscr C$ and $\mathscr D$, b) they are in the same cluster in $\mathscr C$ but in different clusters in $\mathscr D$, c) they are in different clusters in $\mathscr C$ but in the same cluster in $\mathscr D$, and d) they are in different clusters both in $\mathscr C$ and $\mathscr D$. Following \cite{BS08}, the number of observations in each of these classes can be computed from $\mathbf N$, respectively, as follows:
\begin{align*}
a&=\frac{\big(\sum_{i=1}^r\sum_{j=1}^sn_{ij}^2\big)-n}2,\\
b&=\frac{\sum_{i=1}^rn_{i+}^2-\sum_{i=1}^r\sum_{j=1}^sn_{ij}^2}2,\\
c&=\frac{\sum_{j=1}^sn_{+j}^2-\sum_{i=1}^r\sum_{j=1}^sn_{ij}^2}2,\\
d&=\frac{\big(\sum_{i=1}^r\sum_{j=1}^sn_{ij}^2\big)+n^2-\sum_{i=1}^rn_{i+}^2-\sum_{j=1}^sn_{+j}^2}2.
\end{align*}

Many popular agreement indices are defined in terms of $a,b,c,d$, as for instance, the Rand index, the adjusted Rand index, the Jaccard index or the Fowlkes-Mallows index \cite[see, e.g.,][for details]{SHB15}. Moreover, \cite{M92} noted that, if the marginals are given and fixed, then all these quantities depend only on $Q=\sum_{i=1}^r\sum_{j=1}^sn_{ij}^2$, in a way such that $a$ and $d$ are maximized and $b$ and $c$ are minimized when $Q$ is maximized. As a consequence, \cite{BS08} noted that the problem of finding the extrema of the aforementioned indices, and also many others included in \cite{ANM06}, reduces to that of finding the extrema of $Q$, subject to the given marginals.

%As the information registered in the confusion matrix is usually too abundant to handle, it is useful to resort to some statistic to summarize it. It is common to use similarity indices, which attain a value of 1 is both clusterings are identical. More precisely, many of these similarity indices are based on evaluating the concordance of cluster assignments of pairs of observations. A widely used member of this class of pair-based similarity indices is the Rand index, defined as the proportion of observation pairs that are either put together in the same cluster both in $\mathscr C$ and $\mathscr D$, or put in different clusters both in $\mathscr C$ and $\mathscr D$.

This paper deals with a simplified version of the problem, namely the case $r=s=2$. In such a context, the class $\mathcal N(x,y,n)$ of possible contingency tables given the marginals $n_{1+}=x$ and $n_{+1}=y$ becomes uniparametric, with entries and marginals
\begin{table}[h!t]
\centering
\begin{tabular}{c|c|c||c}%\hline
&$D_1$&$D_2$&Total\\\hline
$C_1$ &$k$  &$x-k$&$x$\\\hline
$C_2$ &$y-k$&$n+k-x-y$&$n-x$\\\hline\hline
Total&$y$&$n-y$&$n$
\end{tabular}
%\caption{General form of a $2\times2$ contingency table with given marginals.}
%\label{tab:iris}
\end{table}

\noindent The only free parameter in the previous table is $k$, which must be a non-negative integer. Moreover, all the entries in the table must be non-negative, resulting in the condition
\begin{equation}\label{eq:klim}
\max\{0,x+y-n\}\leq k\leq\min\{x,y\}.
\end{equation}
Thus, the problem reduces to finding the extrema of $Q$ over $\mathcal N(x,y,n)$; that is, finding the minimizer(s) and maximizer(s) of
\begin{equation}\label{eq:Qk}
Q(k)=k^2+(x-k)^2+(y-k)^2+(n+k-x-y)^2
\end{equation}
subject to (\ref{eq:klim}).

It is quite useful to consider some reductions that greatly simplify the problem. Notice that the codification of the cluster labels is arbitrary, for in cluster analysis what matters is the group members, and not the group denomination. This means that we can assume that $x\geq n-x$ since, if that is not the case, it suffices to interchange the labels of clusters $C_1$ and $C_2$. Similarly, we can also assume that $y\geq n-y$. Finally, by interchanging the clusterings $\mathscr C$ and $\mathscr D$, if necessary, it is possible to assume that $x\leq y$. The three reductions can be summarized in the condition $n/2\leq x\leq y$ and, under such a condition, the range (\ref{eq:klim}) of possible values of $k$ simplifies to
\begin{equation}\label{eq:klim2}
x+y-n\leq k\leq x.
\end{equation}
Notice also that, in order to have two clusters in each clustering (i.e., to avoid the possibility of a degenerate, empty cluster), it must be $\max\{x,y\}<n$. With this background, we are ready to state our main result.

Denote by $\lfloor z\rfloor$ and $\lceil z\rceil$ the floor and ceiling of a real number $z$, respectively; that is, $\lfloor z\rfloor$ is the greatest integer less than or equal to $z$, and $\lceil z\rceil$ is the least integer greater than or equal to $z$. Similarly, denote by $\{z\}=z-\lfloor z\rfloor$ the fractional (or decimal) part of $z$. With this notation, the closest integer to $z$ is $\lfloor z+1/2\rfloor$, assuming a round-up tie-breaking rule for those numbers with $\{z\}=1/2$.

\begin{theorem}\label{thm}
The maximum and minimum values of $Q(k)$, given $n$, $n_{1+}=x$ and $n_{+1}=y$, with $n/2\leq x\leq y$, are attained as follows:
\begin{itemize}
\item[a)] If $x>n/2$, the maximum is attained for $k=x$. If $x=n/2$, the maximum is attained both for $k=x+y-n$ and for $k=x$.
\item[b)] If $x+y>3n/2$, the minimum is attained for $k=x+y-n$. If $x+y\leq 3n/2$, the minimum is attained for $k=\lfloor v+1/2\rfloor$ if $\{v\}\neq1/2$ and for both $k=\lfloor v\rfloor$ and $k=\lceil v\rceil$ if $\{v\}=1/2$, where $v=(2x+2y-n)/4$.
\end{itemize}
\end{theorem}

\begin{proof}
The function $Q(k)$ in (\ref{eq:Qk}) is quadratic in $k$. In fact, it can be alternatively expressed as
$$Q(k)=4k^2-2(2x+2y-n)k+(n-x)^2+(n-y)^2+(x+y)^2-n^2,$$
so it is a convex parabola with minimum at $v=(2x+2y-n)/4$. The different cases correspond to the location of $v$ with respect to the lower and upper bounds for $k$ in (\ref{eq:klim2}). It can never be $v\geq x$, since that would entail $y\geq x+n/2\geq n$, which is not possible. So only two possibilities need to be studied: either $v<x+y-n$ or $x+y-n\leq v< x$.

Condition $v<x+y-n$ is equivalent to $x+y>3n/2$. In that case, since $Q(k)$ is a convex parabola with minimum to the left of the possible range of $k$ values, its minimum over (\ref{eq:klim2}) is attained for the leftmost feasible value of $k$, that is $k=x+y-n$, and its maximum for the rightmost one, $k=x$.

When $x+y-n\leq v< x$, which is equivalent to $x+y\leq 3n/2$ (under the remaining conditions), the maximum over (\ref{eq:klim2}) could be attained either for the leftmost or the rightmost range value, depending on the position of $v$ with respect to the midpoint of the range (take into account that $Q(k)$ is symmetric with respect to $v$). But $x\geq n/2$ implies that $v\leq (2x+y-n)/2$, so $v$ is always less than or equal than the midpoint of the of range of $k$ values. If the inequality is strict, then the maximum of $Q(k)$ is attained for the rightmost value $k=x$, and if $x=n/2$ the parabola $Q(k)$ attains the same maximum value for the two range bounds. 

Regarding the minimum, since in this case $v$ is sandwiched between two integer values, the minimum of $Q(k)$ over (\ref{eq:klim2}) is attained for the integer that is closest to $v$, or for the two closest integers when $\{v\}=1/2$, as announced in the statement of the theorem.
\end{proof}

\section{Examples}\label{sec:3}

%\subsection{Maximum and minimum agreement configurations}

It is instructive to visualize the configuration of the confusion matrices for which maximum and minimum agreement is attained, as learned from Theorem \ref{thm}. Recall that our problem reductions imply that $n/2\leq x\leq y$.

Under the given conditions, the maximum agreement between clusterings is always attained for $k=x$, that is, for the confusion matrix
$$\begin{pmatrix}x&0\\y-x&n-y\end{pmatrix}.$$
This corresponds to a situation where the biggest cluster of $\mathscr C$ is completely contained in the biggest cluster of $\mathscr D$, so that the smallest cluster of $\mathscr D$ only contains elements from the smaller cluster of $\mathscr D$, as depicted in Figure \ref{fig:max}. When $x=n/2$, the two clusters in $\mathscr C$ have the same size, so the maximum agreement is attained when the biggest cluster of $\mathscr D$ completely contains either $C_1$ or $C_2$. %The two confusion matrices representing maximum agreement in this case are
%$$\begin{pmatrix}n/2&0\\y-n/2&n-y\end{pmatrix},\qquad\begin{pmatrix}y-n/2&n-y\\n/2&0\end{pmatrix}.$$

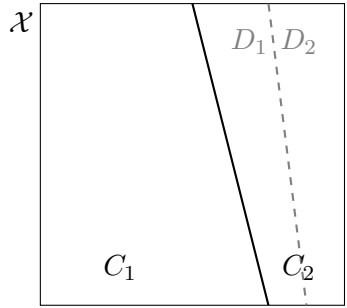
\begin{figure}[t]\centering
\begin{tikzpicture}[scale=0.5]
    \draw (-4,-4) -- (4,-4) -- (4,4) -- (-4,4) -- cycle;
    %\fill [fill=green!60] (0,-4) -- (4,-4) -- (4,4) -- (0,4) -- cycle;
    \draw [thick] (0,4) -- (2,-4);
    \draw [color=gray,thick,dashed] (2,4) -- (3,-4);
    %\draw (0,0) node[] {$\bullet$};
    %\draw (0,0) node[below] {\small $\bx_0$};
    \draw[color=gray] (1.5,3) node {$D_1$};
    \draw (-1.9,-3) node {$C_1$};
    \draw[color=gray] (2.8,3) node[] {$D_2$};
    \draw (2.8,-3) node {$C_2$};
    \draw (-4.5,3.5) node {$\mathcal X$};
\end{tikzpicture}
\caption{A configuration with maximum agreement. The square represents the whole data set $\mathcal X$. Clustering $\mathscr C$ has boundary and cluster tags in black, while clustering $\mathscr D$ has them in grey.} \label{fig:max}
\end{figure}

On the other hand, the situation of minimum agreement is not so straighforward to describe. The condition $x+y>3n/2$ is equivalent to $y>(n-x)+n/2$, so that the biggest cluster in $\mathscr D$ is big enough to contain the smallest cluster in $\mathscr C$ plus a considerable number of observations from the other cluster in $\mathscr C$ (more than half the total sample size). Then, the minimum agreement is attained for the confusion matrix
$$\begin{pmatrix}x+y-n&n-y\\n-x&0\end{pmatrix}.$$
This means that the smallest cluster in $\mathscr C$ is completely contained in the biggest cluster of $\mathscr D$, which thus contains the most heterogeneous possible mixture from members of the two clusters of $\mathscr C$. This is represented graphically in Figure \ref{fig:min1}. To describe the situation for $x+y\leq3n/2$, let us assume for simplicity that $v=(2x+2y-n)/4$ is a integer number, so that the minimum is precisely attained for $k=v$. The confusion matrix for this case is
$$\begin{pmatrix}(2x+2y-n)/4&(n+2x-2y)/4\\(n-2x+2y)/4&(3n-2x-2y)/4\end{pmatrix},$$
but it does not seem easy to find an intuitive description for this situation.

\begin{figure}[t]\centering
\begin{tikzpicture}[scale=0.5]
    \draw (-4,-4) -- (4,-4) -- (4,4) -- (-4,4) -- cycle;
    %\fill [fill=green!60] (0,-4) -- (4,-4) -- (4,4) -- (0,4) -- cycle;
    \draw [thick] (0,4) -- (2,-4);
    \draw [color=gray,thick,dashed] (-2.8,4) -- (-1.8,-4);
    %\draw (0,0) node[] {$\bullet$};
    %\draw (0,0) node[below] {\small $\bx_0$};
    \draw[color=gray] (-2,3) node {$D_1$};
    \draw (-1.9,-3) node {$C_1$};
    \draw[color=gray] (-3.4,3) node[] {$D_2$};
    \draw (2.8,-3) node {$C_2$};
    \draw (-4.5,3.5) node {$\mathcal X$};
\end{tikzpicture}
\caption{A configuration with minimum agreement. The square represents the whole data set $\mathcal X$. Clustering $\mathscr C$ has boundary and cluster tags in black, while clustering $\mathscr D$ has them in grey.} \label{fig:min1}
\end{figure}
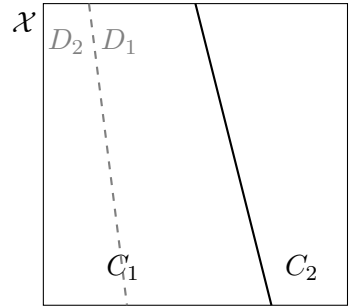

\section{Conclusions}\label{sec:4}

This paper provides an explicit solution for the maximization and minimization of a certain class of agreement measures between two clusterings, given the sizes of their clusters. This problem was posed 35 years ago by \cite{HA85}, and until now it was possible to solve it only through numerical algorithms \citep{SHB15}.

Here, the focus is on the simplest case where each of the two clusterings has only two clusters. Although unfortunately an explicit solution is not yet available in its greatest generality, it should be hoped that our revelation of the explicit forms of the configurations attaining the agreement extremes could be inspiring to tackle the problem with clusterings of arbitrary size. 

For instance, after additional inspection of the results from exhaustive computation of all possible $3\times 3$ confusion matrices with some given marginals, it appears that, to achieve maximum agreement, it is necessary that the biggest cluster in one of the clusterings is completely contained in some cluster in the other clustering.

\bibliographystyle{apalike}

\end{document}